\documentclass{article}
\usepackage{ijcai21}

\usepackage{times}
\usepackage{amssymb,amsmath,graphics,wrapfig,url}

\newtheorem{theorem}{Theorem}
\newtheorem{lemma}{Lemma}
\newtheorem{claim}{Claim}

\newtheorem{definition}{Definition}

\renewcommand{\phi}{\varphi}
\renewcommand{\epsilon}{\varepsilon}

\newenvironment{proof}{\noindent{\sc Proof.}}{\hfill $\boxtimes\hspace{2mm}$\linebreak}
\newcommand{\qed}{\hfill $\boxtimes\hspace{1mm}$}

\newenvironment{proof-of-claim}{\noindent{\sc Proof of Claim.}}{\hfill $\boxtimes\hspace{2mm}$\vspace{0mm}\linebreak}

\begin{document}

\title{As Long as there are Bullets in our Guns:\\ Resource-Constrained Strategies}

\title{Budget-Constrained Coalitions with Discounting}

\title{Budget-Constrained Coalition Strategies with Discounting}


\author{
Lia Bozzone$^1$
\And 
Pavel Naumov$^2$
\affiliations
$^1$Vassar College\\
$^2$King's College\\
\emails
lbbozzone@gmail.com,
pgn2@cornell.edu
}

\maketitle

\begin{abstract}
Discounting future costs and rewards is a common practice in accounting, game theory, and machine learning. In spite of this, existing logics for reasoning about strategies with cost and resource constraints do not account for discounting. The paper proposes a sound and complete logical system for reasoning about budget-constrained strategic abilities that incorporates discounting into its semantics. 
\end{abstract}

\section{Introduction}\label{introduction section}




Several logical systems for reasoning about agent and coalition power in game-like settings have been previously proposed. Among them are coalition logics~\cite{p01illc,p02}, ATL~\cite{ahk02}, ATEL~\cite{vw03sl}, ATLES~\cite{whw07tark}, know-how logics~\cite{aa16jlc,w17synthese,nt17aamas,fhlw17ijcai,nt18ai,nt18aaai,nt19ai}, and STIT~\cite{bp90krdr,x95jlc,h01}. Some of these systems have been extended to incorporate resources and costs of actions~\cite{alnr11jlc,cn20ai,alnr11jlc,cn17ijcai,al18aamas,adl16ijcai,mnp11entcs,alnr17jcss,alnrm15aamas}. Even in the case of multi-step actions, these systems treat current and future costs equally.

At the same time, in game theory, accounting, and machine learning, costs of multi-step transitions are often discounted to reflect the fact that future costs and earnings have lesser present values. Thus, there is a gap between the way resources and costs currently are treated in logic and the way they are accounted for in other fields. To address this gap, in this paper we propose a sound and complete logic of coalition power whose semantics incorporates discounting. Although we formulate our work in terms of cost, it could be applied to any other resource measured in real numbers. It can also be straightforwardly extended to vectors of real numbers to incorporate multiple resources.

As an example, consider a single-player game depicted in Figure~\ref{intro-two-states figure}. This game has four game states $w$, $u$, $v$, and $s$ and a single terminal state $t$. Propositional variable $p$ is true in game states $w$, $u$, and $v$ and is false in game state $s$. We assume that the values of propositional variables are not defined in the terminal state
$t$. The agent $a$ has multiple actions in each game state. These actions are depicted in Figure~\ref{intro-two-states figure} using directed edges. The cost of each action to agent $a$ is shown as a label on the directed edge. For instance, the directed edge from state $w$ to state $u$ with label $2$ means that the agent $a$ has an action with cost $2$ to transition the game from state $w$ to state $u$. Transitioning to the terminal state $t$ represents the termination of the game.

\begin{figure}[ht]
\begin{center}
\vspace{0mm}
\scalebox{0.5}{\includegraphics{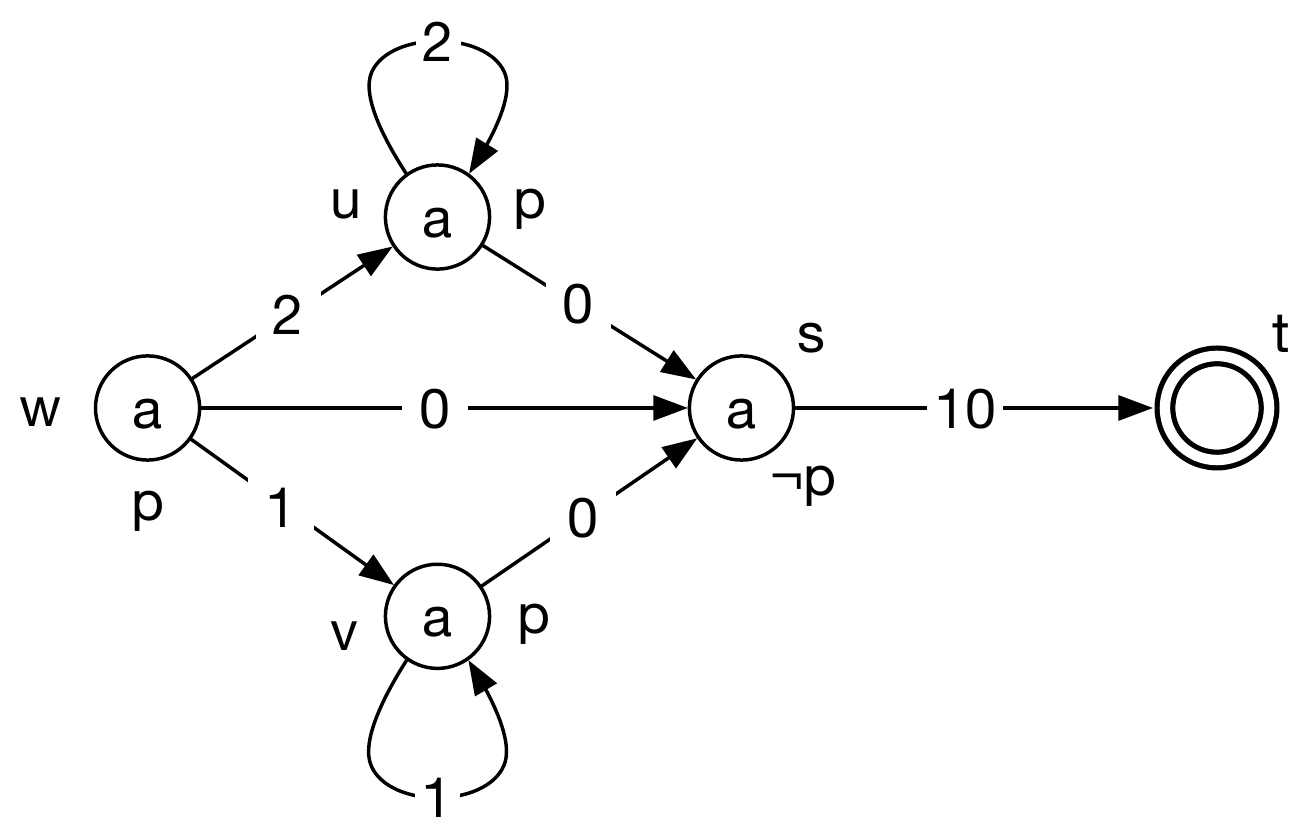}}
\caption{A game.}\label{intro-two-states figure}
\vspace{0mm}
\end{center}
\end{figure}

Note that in state $w$ the agent has two strategies to maintain condition $p$ indefinitely. The first strategy consists of transitioning the game to state $u$ at cost $2$ and then repeatedly applying the action with cost $2$ to keep the game in state $u$. Without discounting, the cost of this strategy is $2+2+2+\dots =+\infty$. The agent also has another strategy to maintain condition $p$ that consists in transitioning to state $v$ at cost $1$ and then keeping the game in state $v$ with recurrent cost $1$. Intuitively, the second strategy is less expensive than the first because each step costs half as much. However, formally, the cost of the second strategy without discounting is the same as the first one: $1+1+1+\dots =+\infty$. 

The problem that we observe here is not specific to costs of strategies. A similar situation also appears in repetitive games, accounting, and reinforcement machine learning algorithms based on Markov decision processes. The solution commonly used to resolve this problem is discounting. It consists of counting the cost on the first step at nominal value, the cost on the second step with a discount factor $\gamma\in (0,1)$, the cost on the third step with discount factor $\gamma^2$, etc. With discounting, the total cost of our first strategy is 
$$
2+2\gamma+2\gamma^2+2\gamma^3+\dots =\dfrac{2}{1-\gamma},
$$
while the cost of our second strategy is
$$
1+1\gamma+1\gamma^2+1\gamma^3+\dots =\dfrac{1}{1-\gamma}.
$$
Since $\frac{1}{1-\gamma}<\frac{2}{1-\gamma}$, we can say that with discounting the second strategy is less expensive than the first. In the rest of this paper, we assume a fixed discount factor $\gamma\in (0,1)$.

\section{Outline}

The rest of this paper is structured as follows: In the next section, we introduce a class of games that will later be used to define the semantics our logical system. Section~\ref{syntax section} defines the language of our system. Section~\ref{semantics section} gives the discounting-based semantics of this language. Section~\ref{perfect recall section} shows that the properties of strategies with discounting depend on whether we consider strategies with or without perfect recall. Section~\ref{axioms section} lists and discusses the axioms of our logical system for the strategies with perfect recall. Section~\ref{completeness section} proves the completeness of our system. Section~\ref{conclusion section}  concludes. Additionally, the proof of soundness can be found in the appendix.

\section{Game Definition}

Throughout the paper, we assume a fixed nonempty set of propositional variables and a fixed set of agents $\mathcal{A}$. By a coalition we mean any subset of $\mathcal{A}$. By $X^\mathcal{A}$ we mean the set of all functions from set $\mathcal{A}$ to a set $X$.

The class of games that we consider is specified below. 

\begin{definition}\label{game}
A game is a tuple $(W,t,\Delta,\epsilon,M,\pi)$, where
\begin{enumerate}
    \item $W$ is a set of {\bf\em game states},
    
    \item $t\notin W$ is a {\bf\em terminal} state, by $W^+$ we denote the set of all states $W\cup\{t\}$,
    
    \item $\Delta$ is an arbitrary set called {\bf\em domain of actions},
    \item $\epsilon\in\Delta$ is a {\bf\em zero-cost action},
    \item $M\subseteq W\times \Delta^{\mathcal{A}}\times [0,\infty)^\mathcal{A}\times W^+$
    is a relation called {\bf\em mechanism}, such that 
    \begin{enumerate}
        \item for each tuple $(w,\delta,u,w')\in M$ and each agent $a\in\mathcal{A}$, if $\delta(a)=\epsilon$, then $u(a)=0$,
        \item for each state $w\in W$ and each {\bf\em complete action profile} $\delta\in \Delta^\mathcal{A}$, there is a function $u\in [0,+\infty)^\mathcal{A}$ and a state $w'\in W^+$ such that $(w,\delta,u,w')\in M$,
    \end{enumerate}
    
    \item $\pi$ is a {\bf\em valuation function} that maps propositional variables into subsets of $W$.
\end{enumerate}
\end{definition}
Intuitively, mechanism is a set of all quadruples $(w,\delta,u,v)$ such that the game might transition from state $w$ to state $v$ under action profile $\delta$ at costs to the individual agents specified by function $u$. 

The defined above games are similar to resource-bounded action frames, which are the semantics of Resource-Bounded Coalition Logic (RBCL) ~\cite{alnr11jlc}. In particular, both of them have a zero-cost action.  

However, there are several differences between these two classes of models. Unlike RBCL frames, our games have only one resource that we call ``cost''. We do this for the sake of presentation simplicity. Multiple resources could be incorporated into our system without any significant changes to the results in this paper. RBCL allows only non-negative {\em integer} resource requirements, while costs in our games are non-negative {\em real} numbers. RBCL assigns a unique cost to each {\em action}, while our games assign cost to each {\em transition}. As a result, the cost to the agent in our setting depends not only on the action of the agent but also on the actions of the other agents. This is similar to how the utility function of an agent in game theory is a function of the complete action profile, not just of the action of that agent. We achieve this by including the cost of the transition for each agent as the third component of a tuple from the mechanism relation. 

Furthermore, the RBCL frames are deterministic while our games are not deterministic because we represent mechanism as a relation, not a function. Unlike RBCL frames, our games can be terminated by the agents. 
In order to make our semantics more general, the games are terminated through a transition to a terminal state. Such transitions allow the agents to be charged upon the termination of a game. 
This ability is significant for our proof of completeness.

\begin{figure}[ht]
\begin{center}
\vspace{0mm}
\scalebox{0.5}{\includegraphics{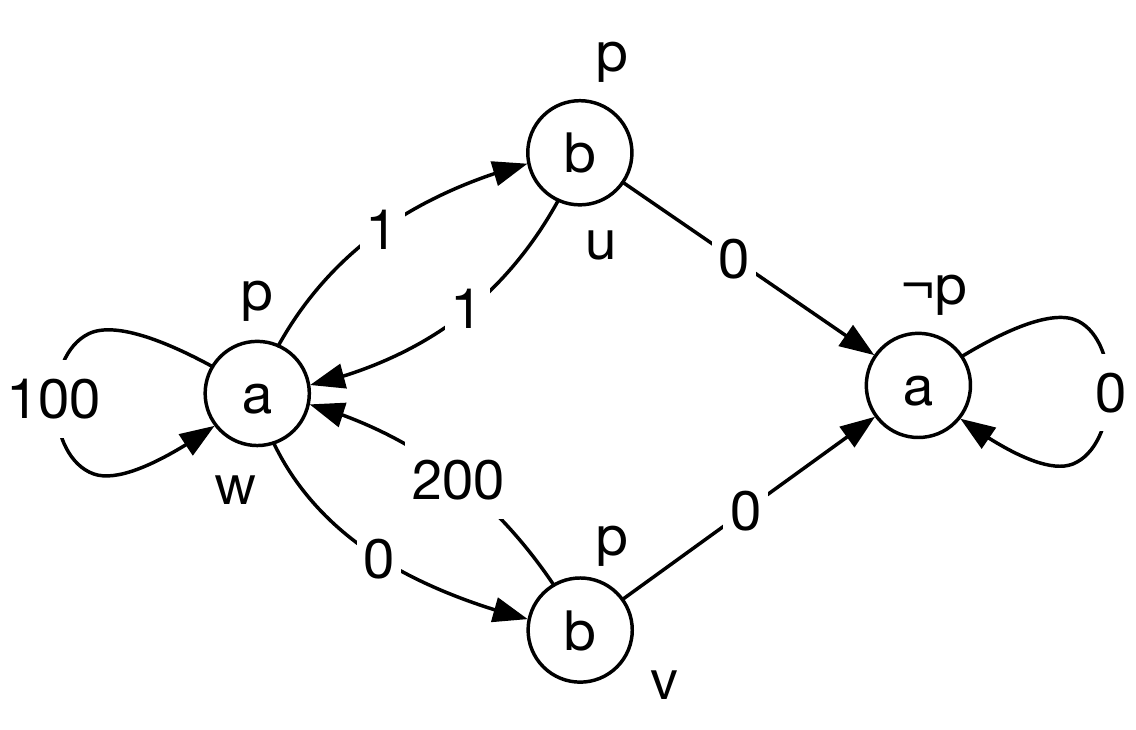}}
\caption{A game. The unreachable terminal state $t$ is not shown in the diagram.}\label{intro-coalition-helps figure}
\vspace{0mm}
\end{center}
\end{figure}

The game depicted in Figure~\ref{intro-two-states figure} has only one player. Figure~\ref{intro-coalition-helps figure} depicts a two-player game. Note that our general notion of the game captured in Definition~\ref{game} allows each agent to influence the outcome of each transition and imposes a cost for each transition on each agent. For the sake of simplicity, the game depicted in Figure~\ref{intro-coalition-helps figure} designates a single ``dictator'' agent in each state. For example, in state $w$, the dictator is agent $a$. The dictator is solely responsible for the choice of the next state and bears all the costs associated with the transition. 
In the diagram, the dictator is specified inside each state's circle. Note that in state $w$ agent $a$ has a strategy to maintain condition $p$ at cost $100+100\gamma+100\gamma^2+\dots=\frac{100}{1-\gamma}$ to herself and cost $0+0\gamma+0\gamma^2+\dots=0$ to agent $b$.

\begin{definition}\label{play}
A play in game $(W,t,\Delta,\epsilon,M,\pi)$ is a finite sequence $w_0,\delta_0,u_0,w_1,\dots,\delta_{n-1},u_{n-1},w_n$ such that
\begin{enumerate}
    \item $w_i\in W$ for  $0\le i< n$ and $w_n\in W^+$,
    \item $\delta_i\in \Delta^\mathcal{A}$, where  $0\le i < n$,
    \item $u_i\in [0,\infty)^\mathcal{A}$ is a {\bf\em cost function}, where $0\le i < n$,
    \item $(w_{i-1},\delta_{i-1},u_{i-1},w_i)\in M$, where  $1\le i\le n$.
\end{enumerate}
\end{definition}
The set of all plays of a given game is denoted by $Play$.

\section{Syntax}\label{syntax section}

The language $\Phi$ of our system is defined by the grammar
$$
\phi:= p \;|\; \neg\phi\;|\;\phi\to\phi\;|\;[C]_x\phi, 
$$
where $p$ is a propositional variable, $C$ is a coalition, and $x$ is a ``constraint'' function from set $C$ to $[0,+\infty)$. We read $[C]_x\phi$ as ``coalition $C$ has  a strategy to maintain condition $\phi$ at individual cost no more than $x(a)$ to each member $a\in C$''. 

If $C$ is a coalition $\{a_1,\dots,a_n\}$ and $x$ is a function from set $C$ to $[0,+\infty)$ such that $x(a_i)=x_i$ for each $i\le n$, then we will use shorthand notation $[a_1,\dots,a_n]_{x_1,\dots,x_n}\phi$ to refer to formula $[C]_x\phi$.

\begin{definition}\label{divide}
For any real $\mu>0$ and any formula $\phi\in \Phi$, formula $\phi/\mu$ is defined recursively as follows:
\begin{enumerate}
    \item $p/\mu \equiv p$, for any propositional variable $p$,
    \item $(\neg\phi)/\mu \equiv\neg(\phi/\mu)$,
    \item $(\phi\to\psi)/\mu\equiv(\phi/\mu)\to(\psi/\mu)$,
    \item $([C]_x\phi)/\mu\equiv[C]_{x/\mu}(\phi/\mu)$.
\end{enumerate}
\end{definition}
For example, $([a,b]_{4,6}\neg[b,c]_{8,2}\,p)/2=[a,b]_{2,3}\neg[b,c]_{4,1}\,p$.

\section{Semantics}\label{semantics section}

In this section we define the semantics of our logical system.

\begin{definition}
An action profile of a coalition $C$ is a function from set $C$ to set $\Delta$. 
\end{definition}

\begin{definition}\label{strategy}
A strategy of a coalition $C$ is a function from set $C\times Play$ to set $\Delta$. 
\end{definition}

Note that each strategy takes into account not just the current state but the whole play. Thus, the strategies that we consider are {\em perfect recall} strategies. We will discuss this in detail in the next section.

\begin{definition}\label{play satisfies strategy}
A play $w_0,\delta_0,u_0,w_1,\dots,u_{n-1},w_n\in Play$ satisfies strategy $s$ of a coalition $C$ if for each $i$ such that $0\leq i < n$ and each agent $a\in C$,
$$\delta_i(a)=s(a,(w_0,\delta_0,u_0,w_1,\dots,u_{i-1},w_{i})).$$
\end{definition}

For any functions $x$ and $y$, we write $x\le_C y$ if $x(a)\le y(a)$ for each $a\in C$. We define notation $x=_C y$ similarly.

\begin{definition}\label{sat}
For each formula $\phi\in\Phi$ and each state $w\in W$ of a game $(W,t,\Delta,\epsilon,M,\pi)$, satisfaction relation $w\Vdash\phi$ is defined recursively as follows:
\begin{enumerate}
    \item $w\Vdash p$, if $w\in \pi(p)$,
    \item $w\Vdash\neg\phi$, if $w\nVdash\phi$,
    \item $w\Vdash\phi\to\psi$, if $w\nVdash\phi$ or $w\Vdash\psi$,
    \item $w\Vdash [C]_x\phi$ if there is a strategy $s$ of coalition $C$ such that for any play $w_0,\delta_0,u_0,w_1,\dots,u_{n-1},w_n\in Play$ that satisfies strategy $s$, if $w=w_0$, then 
    \begin{enumerate}
        \item $\sum_{i = 0}^{n-1}u_i\gamma^i\le_C x$ and
        \item if $w_n\neq t$, then $w_n\Vdash \phi/\gamma^n$.
    \end{enumerate}
\end{enumerate}
\end{definition}

To understand why item 4(b) of the above definition uses formula $\phi/\gamma^n$ instead of formula $\phi$, let us consider an example of a formula $\phi\equiv [D]_y\psi$.   Note that formula $[C]_x[D]_y\psi$ states that coalition $C$  can maintain at cost $x$ the ability of coalition $D$ to maintain $\psi$ at cost $y$. Consider the hypothetical case where $C$, at cost $x$ to $C$, will be  maintaining this ability of $D$ for, say, $10$ transitions. The formula $[C]_x[D]_y\psi$ states that after $10$ moves coalition $D$ should be able to take over and maintain condition $\psi$ at cost $y$ to $D$. Given that in our setting the costs are discounted, an important question is whether $y$ is measured in {\em today's} money or {\em future} money. Note that $y$ in future money is $y\gamma^{10}$ in today's money. On the other hand, $y$ in today's money is $y/\gamma^{10}$ in future money. In this paper we decided to measure all costs in today's money. Thus, cost $y$ in $[C]_x[D]_y\psi$ refers to costs in today's money (in state $w_0$ of Definition~\ref{sat}). In future money (in state $w_n$), the same cost is $y/\gamma^n$.
As a result, item 4(b) of Definition~\ref{sat} uses formula $\phi/\gamma^n$ instead of just $\phi$.

Consider again the game depicted in Figure~\ref{intro-coalition-helps figure}. As discussed earlier, in state $w$, single-agent coalition $\{a\}$ has a strategy to maintain condition $p$ by looping in state $w$ at recurrent cost $100$. The total cost of this strategy is $100+100\gamma+100\gamma^2+\dots=\frac{100}{1-\gamma}$. Thus, $w\Vdash [a]_{100/(1-\gamma)}\,p$. In the same game, single-agent coalition $\{b\}$ also has a strategy to maintain condition $p$. The strategy consists in pushing the game back to state $w$ each time when agent $a$ transitions the game out of state $w$ either into state $u$ or state $v$. The cost of the ``pushing back'' action from state $u$ and $v$ is $1$ and $200$ respectively. Hence, the total cost to agent $b$ could be no more than $0+200\gamma+0+200\gamma^3+0+\dots = 200\gamma/(1-\gamma^2)$. Then, 
$w\Vdash [b]_{200\gamma/(1-\gamma^2)}\,p$. Finally, note that if agents $a$ and $b$ decide to cooperate, then maintaining condition $p$ becomes significantly less expensive for both of them because they can alternate the state of the game between states $w$ and $u$. The total cost of the joint strategy to agent $a$ is  $1+0+\gamma^2+0+\gamma^3+\dots=\frac{1}{1-\gamma^2}$ and to agent $b$ is $0+\gamma+0+\gamma^3+\dots=\frac{\gamma}{1-\gamma^2}$. Therefore, 
$w\Vdash [a,b]_{1/(1-\gamma^2),\gamma/(1-\gamma^2)}\,p$.

\section{Perfect Recall Assumption}\label{perfect recall section}

Definition~\ref{strategy} specifies a strategy of a coalition as a function that assigns an action to each member of a coalition based on a play of the game. In other words, any strategy has access to the whole history of the game rather than just to the current state. Such strategies are often referred to as {\em perfect recall strategies}. As the next example shows,  perfect recall strategies might have different discounted costs than memoryless strategies for the same condition to maintain in the same game. 

\begin{figure}[ht]
\begin{center}
\vspace{0mm}
\scalebox{0.5}{\includegraphics{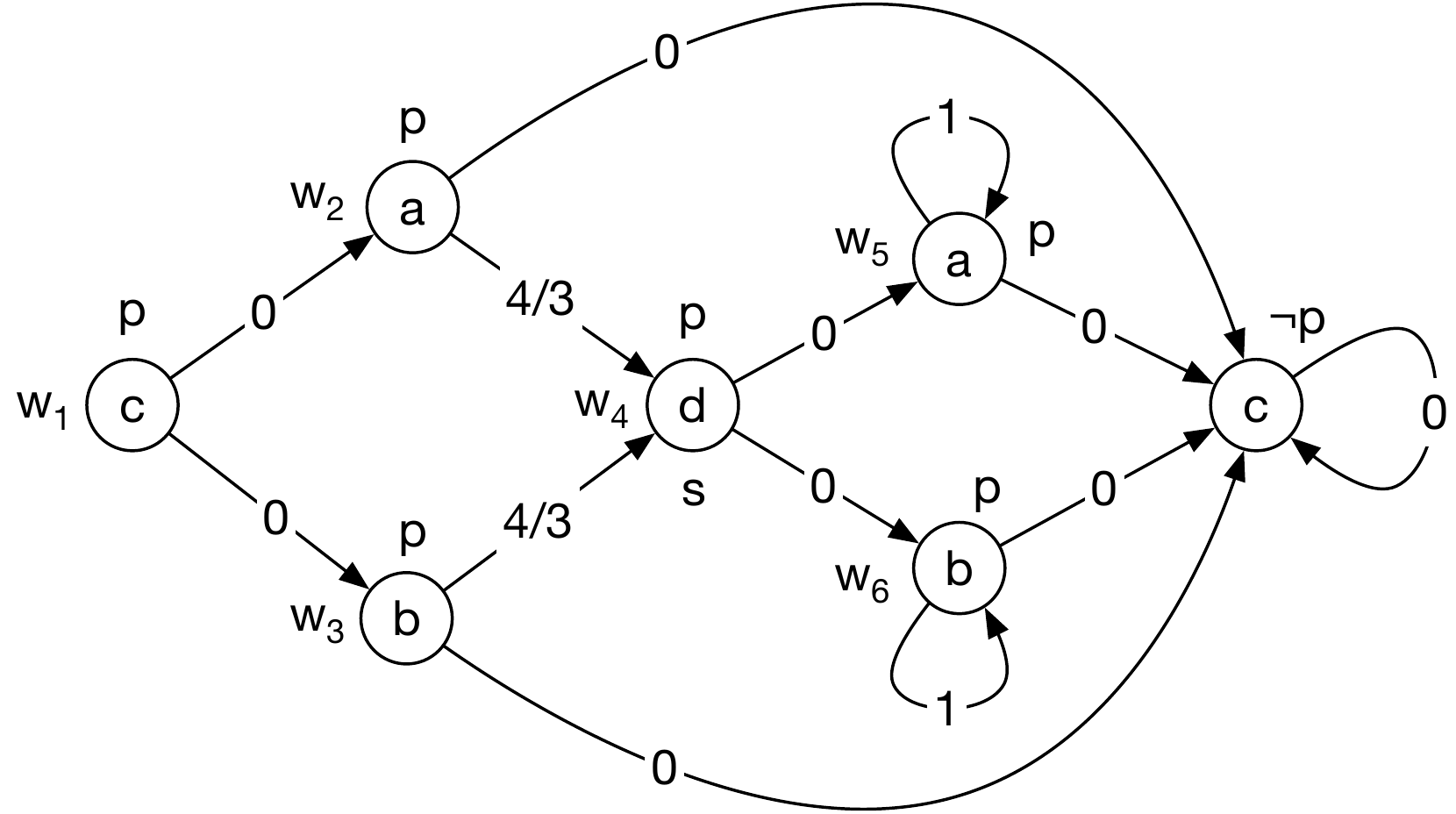}}
\caption{A game. The unreachable terminal state $t$ is not shown in the diagram.}\label{intro-memory-example figure}
\vspace{0mm}
\end{center}
\end{figure}

Consider the game depicted in Figure~\ref{intro-memory-example figure} and assume, for this example only, that $\gamma=2/3$. 
Suppose that coalition $\{a,b,d\}$ wants to maintain condition $p$ starting from state $w_1$. 

Since agent $c$ is not a member of the coalition, the coalition has no control whether the system transitions from state $w_1$ to state $w_2$ or $w_3$. Once the system is either in state $w_2$ or state $w_3$, in order to maintain the condition $p$, agent $a$ or agent $b$, respectively, will have to transition the game to state $w_4$ at cost $\frac{4}{3}\gamma=\frac{4}{3}\cdot\frac{2}{3}=\frac{8}{9}$ to the agent. In state $w_4$, the coalition faces a choice between (i) transitioning game into state $w_5$ in which agent $a$ encounters cost 
$$
1\gamma^3+1\gamma^4+\dots=\dfrac{\gamma^3}{1-\gamma}=\frac{(2/3)^3}{1/3}=\dfrac{8}{9}
$$
to maintain $p$
and (ii) transitioning game into state $w_6$ in which agent $b$ encounters the same cost $1\gamma^3+1\gamma^4+\dots=\frac{8}{9}$ to maintain condition $p$.

If agent $d$ has a perfect recall, then she can balance the costs between agents $a$ and $b$ by transitioning to state $w_6$ if the game transitioned to $w_4$ from state $w_2$ and  transitioning to state $w_5$ if the game transitioned to $w_4$ from state $w_3$. This way, agents $a$ and $b$ encounter the same total costs $8/9$:
$$w_0\Vdash [a,b,d]_{8/9,8/9,0}\, p.$$
At the same time, if agent $d$ does not have memory about the previous state of the game, then either agent $a$ or $b$ might encounter a total cost as high as $8/9+8/9=16/9$ while executing the coalition strategy to maintain condition $p$:
$$w_0\Vdash [a,b,d]_{16/9,16/9,0}\, p.$$

In this paper, we consider discounted costs under perfect recall assumption for all agents. 








\section{Axioms}\label{axioms section}

In this section, we introduce a logical system describing the properties of coalition power modality $[C]_x\phi$. In addition to propositional tautologies in language $\Phi$, the system contains the following axioms:

\begin{enumerate}
    \item Reflexivity: $[C]_x\phi\to\phi$,
    \item Cooperation: if $C\cap D=\varnothing$, then\\ $[C]_x(\phi\to\psi)\to([D]_y\phi\to[C\cup D]_{x\cup y}\psi)$,
    \item Monotonicity: $[C]_x\phi\to [C]_y\phi$, where $x\le_C y$,
    \item Transitivity: $[C]_x\phi\to[C]_x[C]_x\phi$.
\end{enumerate}
Recall that the value of discount factor $\gamma$ has been fixed at the end of Section~\ref{introduction section}. It is worth noting that this factor does not appear explicitly in any of the above axioms.

The Reflexivity axiom says that if coalition $C$ can maintain condition $\phi$ at discounted cost $x$ starting from the current state, then condition $\phi$ must be true in the current state. The Cooperation axiom states that if coalitions $C$ and $D$ are disjoint, coalition $C$ can maintain condition $\phi\to\psi$ at cost $x$, and $D$ can maintain condition $\phi$ at cost $y$, then together they can maintain condition $\psi$ at cost $x\cup y$. Here, by $x\cup y$ we mean the union of two functions with disjoint domains. The Monotonicity axiom states that if a coalition can maintain condition at some cost, then it can maintain the same condition at any larger cost.  

The assumption of the Transitivity axiom states that coalition $C$ has a strategy, say $s$, to maintain condition $\phi$ at cost $x$ in perpetuity. The conclusion states that the same coalition can, at cost $x$, maintain its own ability to maintain $\phi$ at cost $x$. To achieve this, coalition $C$ can use the same strategy $s$. Indeed, assume that coalition $C$ used strategy $s$ for some number of steps at cost $x'\le_C x$ in today's money. Thus, it should be able to keep using it at cost $x-x'\le_C x$ in today's money to maintain $\phi$. Note that it is crucial for this argument that all costs are computed in today's money. Furthermore, the Transitivity axiom is not sound if the cost in the internal modality is measured in future money. A non-trivial proof of soundness of the Transitivity axiom as well as the proofs soundness of all other axioms can be found in the appendix. 

We write $\vdash\phi$ and say that formula $\phi$ is a {\em theorem} if $\phi$ is derivable from the above axioms using the Modus Ponens and the Necessitation inference rules:
$$
\dfrac{\phi,\phi\to\psi}{\psi}
\hspace{15mm}
\dfrac{\phi}{[C]_x\phi}.
$$
In addition to unary relation $\vdash\phi$, we also consider binary relation $X\vdash\phi$. Let $X\vdash\phi$ if formula $\phi$ is provable from {\em the theorems} of our logical system and the set of additional assumptions $X$ using {\em only} the Modus Ponens inference rule. The proofs of the next three auxiliary lemmas can be found in the appendix.

\begin{lemma}\label{superdistributivity lemma} 
If $\phi_1,\dots,\phi_n\vdash\psi$ and sets $C_1$,\dots,$C_n$ are pairwise disjoint, then
$$[C_1]_{x_1}\phi_1,\dots,[C_n]_{x_n}\phi_n\vdash[C_1\cup\dots\cup C_n]_{x_1\cup\dots\cup x_n}\psi.$$
\end{lemma}

\begin{lemma}\label{multiplication lemma} 
If $\phi_1/\gamma,\dots,\phi_n/\gamma\vdash\psi/\gamma$, then
$\phi_1,\dots,\phi_n\vdash\psi$.
\end{lemma}

\begin{lemma}\label{supermonotonicity lemma} 
$\vdash[C]_x\phi\to [D]_y\phi$, where $C\subseteq D$ and $x\le_C y$.
\end{lemma}














\section{Completeness}\label{completeness section}

In this section, we prove the completeness of our logical system. We start the proof by defining the canonical game $(W,t,\Delta,\epsilon,M,\pi)$. The set $W$ is the set of all maximal consistent sets of formulae in language $\Phi$, and $t$ is an arbitrary element such that $t\notin W$. Let $\epsilon$ be an arbitrary element such that $\epsilon\notin\Phi$ and the set of actions $\Delta$ be $\Phi\cup\{\varepsilon\}$. 

\begin{definition}\label{canonical mechanism}
Mechanism $M$ is the set of all quadruples $(w,\delta,u,w')\in W\times \Delta^{\mathcal{A}}\times [0,\infty)^\mathcal{A}\times W^+$ such that if $[C]_x\phi\in w$ and $\delta(a)=[C]_x\phi$ for each agent $a\in C$, then 
\begin{enumerate}
    \item $u\le_C x$ and
    \item if $w'\neq t$, then $([C]_{x-u}\phi)/\gamma\in w'$.
\end{enumerate}
\end{definition}

Informally, action $\delta(a)=[C]_x\phi$ of an agent $a\in C$ means ``as a part of coalition $C$, I request to maintain condition $\phi$ at individual cost $x(b)$ to each member $b\in C$''. In order for the request to be valid, it should be submitted by all members of coalition $C$. Even if all members of coalition $C$ submit the request, it is enforced by the mechanism only if formula $[C]_x\phi$ belongs to the current state $w$. Condition~1 of Definition~\ref{canonical mechanism} stipulates that although the mechanism is free to set the cost $u$ of the transition below what the members of the coalition offered to pay, the mechanism {\em cannot overcharge} them. If the mechanism decides to charge members of the coalition the amount $u$ for transition to state $w'$, then it also must {\em provide the opportunity} for the members to continue to maintain the condition $\phi$ at cost $x-u$. The latter is captured by condition~2 of Definition~\ref{canonical mechanism}.

\begin{definition}\label{canonical pi}
$\pi(p)=\{w\in W\;|\; p\in w\}$. 
\end{definition}

This concludes the definition of the canonical game $(W,t,\Delta,\epsilon,M,\pi)$. As usual, the key step in proving the completeness theorem is an ``induction'' (or ``truth'') lemma, which in our case is Lemma~\ref{induction lemma}. Lemma~\ref{child all} and Lemma~\ref{child exists} below are two auxiliary lemmas that capture the two directions of the induction lemma in the case when formula $\phi$ has the form $[C]_x\psi$.

\begin{lemma}\label{child all}
For each state $w\in W^+$ and each formula $[C]_x\phi\in w$, there is strategy $s$ of coalition $C$ such that, 
for each play $w_0,\delta_0,u_0,w_1,\dots,u_{n-1},w_n$ satisfying strategy $s$, if $w=w_0$, then
\begin{enumerate}
    \item $\sum_{i=0}^{n-1} u_i\gamma^i\le_C x$ and
    \item if $w_n \neq t$, then $\phi/\gamma^n \in w_n$.
\end{enumerate}
\end{lemma}
\begin{proof}
Let action $s(a,\lambda)$ for any agent $a\in C$ and any play
$
\lambda = w_0,\delta_0,u_0,w_1,\dots,u_{n-1},w_n
$ be defined\footnote{Informally, strategy $s$ always requests to maintain condition $\phi$ using remaining budget $x-z$.} as follows:
\begin{equation}\label{choice of s}
\hspace{-2mm}
s(a,\lambda)=
\begin{cases}
([C]_{x-z}\phi)/\gamma^n,&\mbox{if }z\le_C x,\\
\top, &\mbox{otherwise,}
\end{cases}  
\end{equation}
where $z=\sum_{i=0}^{n-1} u_i\gamma^i$.

Consider an arbitrary play $w_0,\delta_0,u_0,w_1,\dots,u_{n-1},w_n$ satisfying strategy $s$ such that $w=w_0$. It will be sufficient to show that conditions 1 and 2 of the lemma hold for this play.
\begin{claim}\label{induction claim}
For each $a\in C$ and each $k$ such that $0\le k\le n$,
\begin{enumerate}
    \item $\sum_{i=0}^{k-1} u_i\gamma^i\le_C x$, 
    \item if $w_k\neq t$, then $\left([C]_{{x-\sum_{i=0}^{k-1} u_i\gamma^i}}\phi\right)/\gamma^k \in w_k$.
\end{enumerate}
\end{claim}
\begin{proof-of-claim}
We prove the claim by induction on integer $k$.
If $k=0$, then
$\sum_{i=0}^{k-1} u_i\gamma^i=0 \le_C x$ by the definition of language $\Phi$ because $[C]_x\phi$ is a formula. Also,
\begin{eqnarray*}
\left([C]_{x-\sum_{i=0}^{k-1} u_i\gamma^i}\phi\right)/\gamma^k
=
\left([C]_{x-0}\phi\right)/\gamma^0
=
[C]_{x}\phi\in w_0
\end{eqnarray*}
by the assumption $[C]_{x}\phi\in w$ of the lemma and the assumption $w=w_0$.

Suppose $k>0$. Then, $(w_{k-1},\delta_{k-1},u_{k-1},w_k)\in M$ by Definition~\ref{play}, the assumption of the lemma that 
$w_0,\delta_0,u_0,w_1$, $\dots,u_{n-1},w_n$ is a play, and the assumption of the claim that $k\le n$. Thus, $w_{k-1}\neq t$ by item 5 of Definition~\ref{game}.
Hence, by the induction hypothesis,
\begin{eqnarray}
&&\sum_{i=0}^{k-2} u_i\gamma^i\le_C x,\label{ind hyp 1}\\
&&\left([C]_{x-\sum_{i=0}^{k-2} u_i\gamma^i}\phi\right)/\gamma^{k-1} \in w_{k-1}.\label{ind hyp 2}
\end{eqnarray}
By Definition~\ref{play satisfies strategy} (step i), equation~(\ref{ind hyp 1}) and equation~(\ref{choice of s}) (step ii), and item~4 of Definition~\ref{divide} (step iii),
\begin{eqnarray}
\delta_{k-1}(a)&\stackrel{\text{i}}{=}&s(a,(w_0,\delta_0,u_0,w_1,\dots,u_{k-2},w_{k-1}))\nonumber\\
&\stackrel{\text{ii}}{=}& ([C]_{x-\sum_{i=0}^{k-2} u_i\gamma^i}\phi)/\gamma^{k-1}\nonumber\\
&\stackrel{\text{iii}}{=}& [C]_{(x-\sum_{i=0}^{k-2} u_i\gamma^i)/\gamma^{k-1}}(\phi/\gamma^{k-1}).\label{delta k -1}
\end{eqnarray}

At the same time, by item 4 of Definition~\ref{divide} (step iv) and equation~(\ref{ind hyp 2}) (step v),
\begin{eqnarray}
&&\hspace{-25mm}    [C]_{(x-\sum_{i=0}^{k-2} u_i\gamma^i)/\gamma^{k-1}}(\phi/\gamma^{k-1}) \nonumber\\
&& \hspace{-10mm}\stackrel{\text{iv}}{=} \left([C]_{{x-\sum_{i=0}^{k-2} u_i\gamma^i}}\phi\right)/\gamma^{k-1}\nonumber\\
&&\hspace{-10mm}\stackrel{\text{v}}{\in} w_{k-1}.\label{w k -1}
\end{eqnarray}
Also, $(w_{k-1},\delta_{k-1},u_{k-1},w_{k})\in M$ by Definition~\ref{play} and the assumption that $w_0,\delta_0,u_0,w_1,\dots,u_{n-1},w_n$ is a play. Thus, by Definition~\ref{canonical mechanism} and statements~(\ref{w k -1}) and (\ref{delta k -1}),
\begin{enumerate}
    \item $u_{k-1}\le_C \left(x-\sum_{i=0}^{k-2} u_i\gamma^i\right)/\gamma^{k-1}$ and
    \item if $w_k\neq t$, then $$([C]_{((x-\sum_{i=0}^{k-2} u_i\gamma^i)/\gamma^{k-1} - u_{k-1})}(\phi/\gamma^{k-1}))/\gamma \in w_{k}.$$
\end{enumerate}
Thus, by the laws of algebra and item 4 of Definition~\ref{divide},
\begin{enumerate}
    \item $u_{k-1}\gamma^{k-1}\le_C x-\sum_{i=0}^{k-2} u_i\gamma^i$ and
    \item if $w_k\neq t$, then $$([C]_{((x-\sum_{i=0}^{k-2} u_i\gamma^i) - u_{k-1}\gamma^{k-1})}\phi)/\gamma^{k} \in w_{k}.$$
\end{enumerate}
The last two statements imply, respectively, parts 1 and 2 of the claim.
\end{proof-of-claim}
The statement of the lemma follows from the above claim when $k=n$. The first part follows immediately. To show the second part, note that by Definition~\ref{divide}, item $2$ of the claim implies $\left([C]_{({x-\sum_{i=0}^{n-1} u_i\gamma^i})/\gamma^n}(\phi/\gamma^n)\right) \in w_n$. Thus, $w_n\vdash \phi/\gamma^n$ by the Reflexivity axiom and the Modus Ponens inference rule. Therefore, $\phi/\gamma^n\in w_n$ because set $w_n$ is maximal.
\end{proof}



\begin{lemma}\label{child exists}
For each state $w\in W$, each formula  $[C]_x\phi\notin w$, and each action profile $\alpha$ of coalition $C$, there is a complete action profile $\delta$, a cost function $u\in [0,+\infty)^\mathcal{A}$, and a state $w'\in W^+$ such that $\alpha=_C\delta$, $(w,\delta,u,w')\in M$, and either (i) $u\not\le_C x$ or (ii) $w'\neq t$ and $\phi/\gamma\notin w'$. 
\end{lemma}

\begin{proof}
Define the complete action profile 
\begin{equation}\label{definition of delta}
\delta(a)=
\begin{cases}
\alpha(a), & \mbox{ if } a\in C,\\
\top, & \mbox{ otherwise}
\end{cases}    
\end{equation}
and cost function\footnote{The choice of function $u$ is perhaps the most unexpected step in our proof. Informally, if agent $a$ is ``bluffing'' and is offering to pay more than $x(a)$, then function $u$ charges the agent the amount she offered to pay, $y(a)$. If the agent makes a ``modest'' offer of no more than $x(a)$, then she is not charged at all.}
$$
u(a)=
\begin{cases}
y(a), & \mbox{if $\alpha(a)=[D]_y\psi$ for some $[D]_y\psi\in\Phi$}, \\
& \mbox{where  $a\in C$ and  $y(a)>x(a)$},\\
0, & \mbox{otherwise}.
\end{cases}
$$

Note that $\alpha=_C\delta$. We consider the following two cases:

\vspace{1mm}
\noindent\textbf{Case I}: $u(a)=0$ for each agent $a\in C$. 
Define set $X$ to be
\begin{eqnarray*}
    X &=& \{\neg(\phi/\gamma)\} \cup \{([D]_{y}\psi)/\gamma\;|\; [D]_y\psi\in w, D\subseteq C, \\
    &&\hspace{30mm}\forall a\in D(\alpha(a)=[D]_y\psi)\}.
\end{eqnarray*}

\begin{claim}
Set $X$ is consistent.
\end{claim}
\begin{proof-of-claim}
Suppose set $X$ is not consistent. Thus, there are formulae
\begin{equation}\label{all about Dpsi}
    [D_1]_{y_1}\psi_1,\dots,[D_n]_{y_n}\psi_n\in w
\end{equation}
such that
\begin{eqnarray}
&&D_1,\dots,D_n\subseteq C,\label{about Ds}\\
&&\alpha(a)=[D_i]_{y_i}\psi_i \hspace{8mm}\forall i\le n \hspace{2mm}\forall a\in D_i,\label{all about alpha}
\end{eqnarray}
and 
\begin{equation}
    ([D_1]_{y_1}\psi_1)/\gamma,\dots, ([D_n]_{y_n}\psi_n)/\gamma \vdash \phi/\gamma.\label{divided by gamma}
\end{equation}
Without loss of generality, we can assume that formulae $([D_1]_{y_1}\psi_1)/\gamma,\dots, ([D_n]_{y_n}\psi_n)/\gamma$ are distinct. Thus, formulae $[D_1]_{y_1}\psi_1,\dots,[D_n]_{y_n}\psi_n$ are also distinct by Definition~\ref{divide}.  Hence, sets $D_1,\dots,D_n$ are pairwise disjoint due to assumption~(\ref{all about alpha}).

By Lemma~\ref{multiplication lemma}, statement~(\ref{divided by gamma}) implies that
$$
[D_1]_{y_1}\psi_1,\dots, [D_n]_{y_n}\psi_n \vdash \phi.
$$
Then, by Lemma~\ref{superdistributivity lemma} and because sets $D_1,\dots,D_n$ are pairwise disjoint, 
\begin{eqnarray*}
&&\hspace{-10mm}[D_1]_{y_1}[D_1]_{y_1}\psi_1,\dots, [D_n]_{y_n}[D_n]_{y_n}\psi_n \\
&&\vdash [D_1\cup\dots\cup D_n]_{y_1\cup \dots \cup y_n}\phi.
\end{eqnarray*}
Thus, by the Transitivity axiom and the Modus Ponens inference rule applied $n$ times,
\begin{eqnarray*}
&[D_1]_{y_1}\psi_1,\dots, [D_n]_{y_n}\psi_n 
\vdash [D_1\cup\dots\cup D_n]_{y_1\cup \dots \cup y_n}\phi.
\end{eqnarray*}
Notice that $y_i(a)\le x(a)$ for any $i\le n$ and any agent $a\in D_i$. Indeed, suppose that $y_i(a)> x(a)$. Hence $u(a)=y_i(a)$ by the choice of cost function $u$ and statements~(\ref{about Ds}) and (\ref{all about alpha}). Thus, $u(a)>x(a)$. Then, $u(a)>0$ because function $x$ is non-negative by the assumption $[C]_x\phi\in\Phi$, which contradicts the assumption $u(a)=0$ of the case.  Hence, by Lemma~\ref{supermonotonicity lemma} and the Modus Ponens inference rule,
$$
[D_1]_{y_1}\psi_1,\dots, [D_n]_{y_n}\psi_n \vdash [C]_{x}\phi.
$$
Then, 
$w \vdash [C]_{x}\phi$ 
by the assumption~(\ref{all about Dpsi}).
Thus, $[C]_{x}\phi\in w$ because set $w$ is maximal, which contradicts the assumption $[C]_{x}\phi\notin w$ of the lemma.
\end{proof-of-claim}

Let $w'$ be any maximal consistent extension of set $X$. Note that $\neg(\phi/\gamma)\in X\subseteq w'$ by the choice of sets $X$ and $w'$. Thus, $\phi/\gamma\notin w'$ because set $w'$ is consistent.

\begin{claim}
$(w,\delta,u,w')\in M$.
\end{claim}
\begin{proof-of-claim}
Consider any formula $[D]_y\psi\in w$ such that
\begin{equation}\label{lost assumption}
    \delta(a)=[D]_y\psi\;\;\;\mbox{for each agent}\;\;\; a\in D.
\end{equation}
By Definition~\ref{canonical mechanism}, it suffices to show that $u\le_D y$ and $([D]_{y-u}\psi)/\gamma\in w'$. We consider the following two cases:

\noindent{\em Case Ia}: $D\subseteq C$. Thus, $\alpha(a)=\delta(a)=[D]_y\psi$ for each agent $a\in D$ by equation~(\ref{definition of delta}) and  assumption~(\ref{lost assumption}). Hence, $([D]_y\psi)/\gamma\in X$ by the choice of set $X$. Then, $([D]_{y-u}\psi)/\gamma\in X$ by the assumption of Case I that $u=_C 0$ and the assumption $D\subseteq C$ of Case Ia. Therefore, $([D]_{y-u}\psi)/\gamma\in w'$ by the choice of set $w'$. Additionally, $u=_D 0\le_D y$ because $0\le_D y$ by the definition of language $\Phi$.

\noindent{\em Case Ib}: There is an agent $a\in D\setminus C$. Hence, $\top=\delta(a)=[D]_y\psi$ by equation~(\ref{definition of delta}) and assumption~(\ref{lost assumption}). Therefore, formula $[D]_y\psi$ is identical to formula $\top$, which is a contradiction.
\end{proof-of-claim}
Note that $\neg(\phi/\gamma)\in X\subseteq w'$ by the choice of sets $X$ and $w'$. Therefore, $\phi/\gamma\notin w'$ because set $w'$ is consistent.

\vspace{1mm}
\noindent\textbf{Case II}: $u(a) \neq 0$ for at least one agent $a\in C$. Thus, $u(a)=y(a)>x(a)$ by the choice of function $u$. Therefore, $u\not\le_C x$. Choose $w'$ to be the terminal state $t$. 

\begin{claim}
$(w,\delta,u,w')\in M$.
\end{claim}
\begin{proof-of-claim}
Consider any formula $[D]_y\psi\in w$ such $\delta(a)=[D]_y\psi$ for each agent $a\in D$.
By Definition~\ref{canonical mechanism} and because $w'=t$, it suffices to show that $u\le_D y$. Recall that $0\le_D y$ because $[D]_y\psi$ is a formula. Therefore, $u\le_D y$ by the choice of function $u$.
\end{proof-of-claim}
This concludes the proof of the lemma.
\end{proof}

The next lemma is usually referred to as an ``induction'' or ``truth'' lemma. It is proven by induction on the structural complexity of formula $\phi$ using Lemma~\ref{child all} and Lemma~\ref{child exists} in the case where formula $\phi$ has the form $[C]_x\psi$. The proof of this lemma can be found in the appendix.

\begin{lemma}\label{induction lemma}
$w\Vdash\phi$ iff $\phi\in w$ for each state $w\in W$ and each formula $\phi\in\Phi$.
\end{lemma}

\begin{theorem}\label{completeness theorem}
If $X\nvdash\phi$, then there is a state $w$ of a game such that $w\Vdash\chi$ for each $\chi\in X$ and $w\nVdash\phi$. 
\end{theorem}
\begin{proof}
Suppose that $X\nvdash\phi$. Let $w$ be any maximal consistent extension of set  $X\cup\{\neg\phi\}$. Note that $w$ is a state of the canonical game. Then, $w\Vdash\chi$ for each $\chi\in X$ and  $w\Vdash\neg\phi$ by Lemma~\ref{induction lemma}. Therefore, $w\nVdash\phi$ by Definition~\ref{sat}.
\end{proof}

\section{Conclusion}\label{conclusion section}

In this paper we proposed a coalition power logic whose semantics incorporates discounting. The main technical result is a strongly sound and strongly complete logical system for coalition strategies with perfect recall. 

\label{end of paper}
\bibliographystyle{named}
\bibliography{sp}

\clearpage

\begin{center}
    {\sc\Large Technical Appendix}
    
    \vspace{5mm}
    
    This appendix is not a part of IJCAI-21 proceedings.
\end{center}

\appendix

\section{Soundness}

In this section we prove soundness of each of our axioms as a separate lemma. In these lemmas we assume that $w$ is an arbitrary game state of a game $(W,t,\Delta,\epsilon,M,\pi)$.

\begin{lemma}
If $w\Vdash [C]_x\phi$, then $w\Vdash \phi$.
\end{lemma}
\begin{proof}
Single-element  sequence $w$ is a play by Definition~\ref{play}. Thus, by item 4 of Definition~\ref{sat}, the assumption $w\Vdash [C]_x\phi$ implies that $w\Vdash \phi/\gamma^0$. Hence, $w\Vdash \phi/1$. Therefore, $w\Vdash \phi$ by Definition~\ref{divide}.
\end{proof}

\begin{lemma}
If $w\Vdash [C]_x\phi$ and $x\le_C y$, then $w\Vdash [C]_y\phi$.
\end{lemma}
\begin{proof}
By item 4 of Definition~\ref{sat}, the assumption $w\Vdash [C]_x\phi$ implies that there is a strategy $s$ of coalition $C$ such that for any play $w_0,\delta_0,u_0,w_1,\dots,u_{n-1},w_n\in Play$ that satisfies strategy $s$, if $w=w_0$, then 
    \begin{enumerate}
        \item $\sum_{i = 0}^{n-1}u_i\gamma^i\le_C x$ and
        \item if $w_n\neq t$, then $w_n\Vdash \phi/\gamma^n$.
    \end{enumerate}
Note that condition 1 above implies that $\sum_{i = 0}^{n-1}u_i\gamma^i\le_C y$ by the assumption $x\le_C y$ of the lemma. Therefore, $w\Vdash [C]_y\phi$ again by item 4 of Definition~\ref{sat}.
\end{proof}

\begin{lemma}
If $w\Vdash [C]_x(\phi\to\psi)$, $w\Vdash [D]_y\phi$, and coalitions $C$ and $D$ are disjoint, then $w\Vdash [C\cup D]_{x\cup y}\psi$.
\end{lemma}
\begin{proof}
By item 4 of Definition~\ref{sat}, the assumption $w\Vdash [C]_x(\phi\to\psi)$ implies that there is a strategy $s_1$ of coalition $C$ such that for any play $w_0,\delta_0,u_0,w_1,\dots,u_{n-1},w_n\in Play$ that satisfies strategy $s_1$, if $w=w_0$, then 
    \begin{enumerate}
        \item $\sum_{i = 0}^{n-1}u_i\gamma^i\le_C x$ and
        \item if $w_n\neq t$, then $w_n\Vdash (\phi\to\psi)/\gamma^n$.
    \end{enumerate}
Similarly, the assumption $w\Vdash [D]_y\phi$ implies that there is a strategy $s_2$ of coalition $C$ such that for any play $w_0,\delta_0,u_0,w_1,\dots,u_{n-1},w_n\in Play$ that satisfies strategy $s_2$, if $w=w_0$, then 
    \begin{enumerate}
        \item[3.] $\sum_{i = 0}^{n-1}u_i\gamma^i\le_D y$ and
        \item[4.] if $w_n\neq t$, then $w_n\Vdash \phi/\gamma^n$.
    \end{enumerate}
Consider strategy $s$ of coalition $C\cup D$ such that 
$$
s(a,\lambda)=
\begin{cases}
s_1(a,\lambda), &\mbox{ if } a\in C,\\
s_2(a,\lambda), &\mbox{ if } a\in D,
\end{cases}
$$
for any play $\lambda\in Play$. Note that strategy $s$ is well-defined because sets $C$ and $D$ are disjoint by the assumption of the lemma. 

Consider any play $w_0,\delta_0,u_0,w_1,\dots,u_{n-1},w_n\in Play$ satisfying strategy $s$. Thus, by Definition~\ref{play satisfies strategy}, this play satisfies strategies $s_1$ and $s_2$. Note that conditions 1 and 3 above imply that $\sum_{i = 0}^{n-1}u_i\gamma^i\le_{C\cup D} x\cup y$. Suppose that $w_n\neq t$. By item 4 of Definition~\ref{sat}, it suffices to show that $w_n\Vdash\psi/\gamma^n$. Indeed,  condition 2  above implies that $w_n\Vdash (\phi\to\psi)/\gamma^n$. Thus, $w_n\Vdash \phi/\gamma^n\to\psi/\gamma^n$ by Definition~\ref{divide}. Therefore, $w_n\Vdash\psi/\gamma^n$ by item 3 of Definition~\ref{sat} and condition 4 above.
\end{proof}

The next auxiliary lemma follows from Definition~\ref{divide}.
\begin{lemma}\label{divide by product}
$\phi/(\gamma\gamma') = (\phi/\gamma)/\gamma'$. \qed
\end{lemma}

\begin{lemma}
If $w\Vdash [C]_x\phi$, then $w\Vdash [C]_x[C]_x\phi$.
\end{lemma}
\begin{proof}
By item 4 of Definition~\ref{sat}, the $w\Vdash [C]_x\phi$ implies that there is a strategy $s$ of coalition $C$ such that for any play $w_0,\delta_0,u_0,w_1,\dots,u_{n-1},w_n\in Play$ that satisfies strategy $s$, if $w=w_0$, then 
\begin{equation}\label{dec7-a}
    \sum_{i = 0}^{n-1}u_i\gamma^i\le_C x
\end{equation}
and
\begin{equation}\label{dec7-b}
    \mbox{if $w_n\neq t$, then $w_n\Vdash \phi/\gamma^n$.}
\end{equation}
Consider any play $w'_0,\delta'_0,u'_0,w'_1,\dots,u'_{m-1},w'_m\in Play$ that satisfies strategy $s$ such that $w=w'_0$. By the same item 4 of Definition~\ref{sat}, it suffices to show that
\begin{equation}\label{dec7-c}
        \sum_{i = 0}^{m-1}u'_i\gamma^i\le_C x
\end{equation}
and
\begin{equation}\label{dec7-d}
        \mbox{if $w'_m\neq t$, then $w'_m\Vdash ([C]_x\phi)/\gamma^m$}.
\end{equation}
Note that statement~(\ref{dec7-c}) follows from assumption~(\ref{dec7-a}). Thus, it is enough to prove statement~(\ref{dec7-d}). Suppose $w_m\neq t$, then, by Definition~\ref{divide}, it suffices to show that
\begin{equation}\label{dec7-e}
    w_m\Vdash [C]_{x/\gamma^m}(\phi/\gamma^m).
\end{equation}
Consider strategy
\begin{eqnarray}
&&s'(a, (w_0,\delta_0,u_0,w_1,\dots,u_{n-1},w_n))=\label{dec7-f}\\
&&=\begin{cases}
s(a, (w'_0,\delta'_0,w'_1,\dots,u'_{m-1},\nonumber\\
\hspace{8mm} w'_m,\delta_0,u_0,\dots,u_{n-1},w_n)), &\mbox{if $w'_m=w_0$}\nonumber\\
\epsilon,&\mbox{otherwise}.\nonumber
\end{cases}
\end{eqnarray}
Consider any play $w''_0,\delta''_0,u''_0,w''_1,\dots,u''_{k-1},w''_k\in Play$ that satisfies strategy $s'$ such that $w_m=w''_0$. By the same item 4 of Definition~\ref{sat}, to prove statement~(\ref{dec7-e}) it suffices to show that 
$$
        \sum_{i = 0}^{k-1}u''_i\gamma^i\le_C x/\gamma^m
$$
and if $w''_k\neq t$, then $w''_k\Vdash ([C]_{x/\gamma^m}(\phi/\gamma^m))/\gamma^k$.
Both of these facts follow from the three claims below.

\begin{claim}\label{dec7-claimA}
The play $w'_0,\delta'_0,u'_0,w'_1,\dots,u'_{m-1},w'_m,\delta''_0,u''_0$,\\ $w''_1, \dots$, $u''_{k-1},w''_k$ satisfies strategy $s$.  
\end{claim}
\begin{proof-of-claim}
The statement of the claim follows from Definition~\ref{play satisfies strategy}, equation~(\ref{dec7-f}), and the assumption that the play $w''_0,\delta''_0,u''_0,w''_1,\dots,u''_{k-1},w''_k$ satisfies strategy $s'$.
\end{proof-of-claim}

\begin{claim}\label{dec7-claimB}
$\sum_{i = 0}^{k-1}u''_i\gamma^i\le_C x/\gamma^m$.
\end{claim}
\begin{proof-of-claim}
By statement~(\ref{dec7-a}) and Claim~\ref{dec7-claimA},
$$
\sum_{i = 0}^{m-1}u'_i\gamma^i + \sum_{i=m}^{m+k-1}u''_{i-m}\gamma^i\le_C x.
$$
Functions $u'_0,\dots,u'_{m-1}$ are non-negative by item 5 of Definition~\ref{game}. Hence,
$
\sum_{i=m}^{m+k-1}u''_{i-m}\gamma^i\le_C x.
$
Thus,
$
\gamma^m\sum_{i=0}^{k}u''_{i}\gamma^i\le_C x.
$
Then, $\sum_{i = 0}^{k-1}u''_i\gamma^i\le_C x/\gamma^m$.
\end{proof-of-claim}

\begin{claim}\label{dec7-claimC}
If $w''_k\neq t$, then $w''_k\Vdash ([C]_{x/\gamma^m}(\phi/\gamma^m))/\gamma^k$.
\end{claim}
\begin{proof-of-claim}
By Claim~\ref{dec7-claimA}, the play $w'_0,\delta'_0,u'_0,w'_1,\dots,$ $u'_{m-1}, w'_m,\delta''_0,u''_0$, $w''_1, \dots$, $u''_{k-1},w''_k$ satisfies strategy $s$. Then, $w''_k\Vdash([C]_x\phi)/\gamma^{m+k}$ by equation~(\ref{dec7-b}) and the assumption $w''_k\neq t$. 
Hence $w''_k\Vdash (([C]_{x}\phi)/\gamma^m)/\gamma^k$
by Lemma~\ref{divide by product}.
Therefore,  $w''_k\Vdash ([C]_{x/\gamma^m}(\phi/\gamma^m))/\gamma^k$
by item~4 of Definition~\ref{divide}.
\end{proof-of-claim}
This concludes the proof of the lemma.
\end{proof}

\section{Auxiliary Lemmas}

\renewcommand*{\thelemma}{\ref{superdistributivity lemma}}
\begin{lemma}
If $\phi_1,\dots,\phi_n\vdash\psi$ and sets $C_1$,\dots,$C_n$ are pairwise disjoint, then
$$[C_1]_{x_1}\phi_1,\dots,[C_n]_{x_n}\phi_n\vdash[C_1\cup\dots\cup C_n]_{x_1\cup\dots\cup x_n}\psi.$$
\end{lemma}
\begin{proof}
Apply the deduction lemma $n$ times to the assumption $\phi_1,\dots,\phi_n\vdash\psi$. Then, $\vdash\phi_1\to(\dots\to(\phi_n\to\psi))$. Thus, $$\vdash[\varnothing]_0(\phi_1\to(\dots\to(\phi_n\to\psi))),$$ by the Necessitation inference rule. Hence, $$\vdash[C_1]_{x_1}\phi_1\to[C_1]_{x_1}(\phi_2\dots\to(\phi_n\to\psi))$$ by the Cooperation axiom and the Modus Ponens inference rule. Then, $$[C_1]_{x_1}\phi_1\vdash[C_1]_{x_1}(\phi_2\dots\to(\phi_n\to\psi))$$ by the Modus Ponens inference rule. Thus, again by the Cooperation axiom and Modus Ponens, $$[C_1]_{x_1}\phi_1\vdash[C_2]_{x_2}\phi_2\to 
[C_1 \cup C_2]_{x_1 \cup x_2}(\phi_3\dots\to(\phi_n\to\psi)).$$
Therefore, by repeating the last two steps $n-2$ times, $[C_1]_{x_1}\phi_1,\dots,[C_n]_{x_n}\phi_n\vdash[C_1 \cup\dots\cup C_n]_{x_1 \cup \dots \cup x_n}\psi$.
\end{proof}

\renewcommand*{\thelemma}{\ref{multiplication lemma}}
\begin{lemma}
If $\phi_1/\gamma,\dots,\phi_n/\gamma\vdash\psi/\gamma$, then
$\phi_1,\dots,\phi_n\vdash\psi$.
\end{lemma}
\begin{proof}
Note that if a sequence of formulae $\chi_1,\dots,\chi_n$ is a derivation in our logical system, then for each real number $\mu>0$, sequence $\chi_1/\mu,\dots,\chi_n/\mu$ is also a derivation. Hence, for any formulae $\phi_1,\dots,\phi_n,\psi\in \Phi$, if $\phi_1,\dots,\phi_n\vdash\psi$, then $\phi_1/\mu,\dots,\phi_n/\mu\vdash\psi/\mu$. Let $\mu=\gamma^{-1}$. Thus, for any formulae $\phi_1,\dots,\phi_n,\psi\in \Phi$, if $\phi_1/\gamma,\dots,\phi_n/\gamma\vdash\psi/\gamma$, then $\phi_1,\dots,\phi_n\vdash\psi$.
\end{proof}

\renewcommand*{\thelemma}{\ref{supermonotonicity lemma}}
\begin{lemma}
$\vdash[C]_x\phi\to [D]_y\phi$, where $C\subseteq D$ and $x\le_C y$.
\end{lemma}
\begin{proof}
Let $y'$ be the restriction of function $y$ to set $D\setminus C$. Then, by the Cooperation axiom and the assumption $C\subseteq D$,
$$\vdash [D\setminus C]_{y'}(\phi\to\phi)\to([C]_x\phi\to [D]_{y'\cup x}\phi).$$
Note that $\phi\to\phi$ is a propositional tautology. Hence, by the Necessitation inference rule $\vdash [D\setminus C]_{y'}(\phi\to\phi)$. Then, $\vdash [C]_x\phi\to [D]_{y'\cup x}\phi$ by the Modus Ponens inference rule. Note that $\vdash [D]_{y'\cup x}\phi\to [D]_{y}\phi$ by the Monotonicity axiom and the assumption $x\le_C y$ of the lemma. Therefore, $\vdash [C]_x\phi\to [D]_{y}\phi$ by propositional reasoning.
\end{proof}

\section{Completeness}

\renewcommand*{\thelemma}{\ref{induction lemma}}
\begin{lemma}
$w\Vdash\phi$ iff $\phi\in w$ for each state $w\in W$ and each formula $\phi\in\Phi$.
\end{lemma}
\begin{proof}
We prove the statement by induction on the structural complexity of formula $\phi$. If $\phi$ is a propositional variable, then the required follows from item~1 of Definition~\ref{sat} and Definition~\ref{canonical pi}. If formula $\phi$ is a negation or an implication, then the statement of the lemma follows from items~2 and 3 of Definition~\ref{sat}, the induction hypothesis, and the maximality and consistency of set $w$ in the standard way.

Suppose that formula $\phi$ has the form $[C]_x\psi$.

\noindent$(\Rightarrow):$ Assume that $[C]_x\psi\notin w$.  Consider any strategy $s$ of coalition $C$. Define action profile $\alpha$ of coalition $C$ such that, for each agent $a\in C$,
\begin{equation}\label{oct19-alpha}
    \alpha(a)=s(a,w_0).
\end{equation}
By  Lemma~\ref{child exists}, there is a complete action profile $\delta$, a cost function $u\in [0,+\infty)^\mathcal{A}$, and a state $w'\in W^+$ such that $\alpha=_C\delta$, $(w,\delta,u,w')\in M$, and 
\begin{equation}\label{oct19-either}
    \mbox{either (i) $u\not\le_C x$ or (ii) $w'\neq t$ and $\psi/\gamma\notin w'$}.
\end{equation}
Consider play $w,\delta,w'$. This play satisfies strategy $s$ by Definition~\ref{play satisfies strategy}, assumption $\alpha=_C\delta$, and equation~(\ref{oct19-alpha}). Then, by item~4 of Definition~\ref{sat}, to prove $w\nVdash [C]_x\psi$, it suffices to show that either (i) $u\not\le_C x$ or (ii) $w'\neq t$ and $w'\nVdash\psi/\gamma$. Note that this statement is true by statement~(\ref{oct19-either}) and the induction hypothesis.

\noindent$(\Leftarrow):$ Suppose that $[C]_x\phi\in w$. By Lemma~\ref{child all},  there is a strategy $s$ of coalition $C$ such that, 
for each play $w_0,\delta_0$, $u_0,w_1,\dots,u_{n-1},w_n$ satisfying strategy $s$, if $w=w_0$, then
\begin{enumerate}
    \item $\sum_{i=0}^{n-1} u_i\gamma^i\le_C x$ and
    \item if $w_n \neq t$, then $\phi/\gamma^n \in w_n$.
\end{enumerate}
Therefore, $w\Vdash[C]_x\phi$ by item 4 of Definition~\ref{sat} and the induction hypothesis.
\end{proof}

\end{document}